\documentclass[11pt]{article}

\usepackage{amsmath,amssymb,amsfonts,amsthm,epsfig}
\usepackage[usenames,dvipsnames]{xcolor}
\usepackage{bm,xspace}
\usepackage{tcolorbox}
\usepackage{cancel}
\usepackage{fullpage}
\usepackage{liyang}
\usepackage{framed}
\usepackage{verbatim}
\usepackage{enumitem}
\usepackage{array}
\usepackage{multirow}
\usepackage{afterpage}
\usepackage{mathrsfs}
\usepackage{pifont} 
\usepackage{chngpage}
\usepackage[normalem]{ulem}
\usepackage{boxedminipage}
\usepackage{caption}
\usepackage{forest}

\def\colorful{1}

\ifnum\colorful=1
\newcommand{\violet}[1]{{\color{violet}{#1}}}

\fi
\ifnum\colorful=0
\newcommand{\violet}[1]{{{#1}}}

\fi

\newcommand{\BayesOpt}{\mathrm{opt}}
\newcommand{\error}{\mathrm{error}}
\newcommand{\round}{\mathrm{round}}
\newcommand{\trunc}{\mathrm{trunc}}

\DeclareMathOperator*{\argmin}{arg\,min}

\newtheorem{manualtheoreminner}{Theorem}
\newenvironment{restatement}[1]{%
  \renewcommand\themanualtheoreminner{#1}%
  \manualtheoreminner
}{\endmanualtheoreminner}

\begin{document}

\title{Learning stochastic decision trees\vspace*{15pt}}

\author{Guy Blanc \vspace{8pt} \\ \hspace{-5pt}{\sl Stanford} \and \hspace{10pt} Jane Lange \vspace{8pt} \\
\hspace{4pt}  {\sl MIT}
\and Li-Yang Tan \vspace{8pt} \\ \hspace{-8pt} {\sl Stanford}}

\date{\vspace{5pt}\small{\today}}

\maketitle

\begin{abstract}

We give a quasipolynomial-time algorithm for learning {\sl stochastic decision trees} that is optimally resilient to adversarial noise.   Given an $\eta$-corrupted set of uniform random samples labeled by a size-$s$ stochastic decision tree, our algorithm runs in time $n^{O(\log(s/\eps)/\eps^2)}$
and returns a  hypothesis with error within an additive $2\eta + \eps$ of the Bayes optimal.  An additive $2\eta$ is the information-theoretic minimum.  

Previously no non-trivial algorithm with a guarantee of $O(\eta) + \eps$ was known, even for weaker noise models.  Our algorithm is furthermore proper, returning a hypothesis that is itself a decision tree; previously no such algorithm was known even in the noiseless setting. 

\end{abstract}
\thispagestyle{empty}

\newpage
\setcounter{page}{1}


\section{Introduction} 

Decision trees are a touchstone class in learning theory.  There is by now a rich and vast literature on the problem of learning decision trees, spanning three decades and studying it in a variety of models and from a variety of perspectives~\cite{EH89,Riv87,Blu92,Han93,Bsh93,KM93,BFJKMR94,HJLT96,KM99,MR02,JS06,OS07,GKK08,Lee09,KS06,KST09,HKY18,CM19,BLT-ICML,BLT-ITCS,BGLT-NeurIPS1,BDM20}.  

We consider the problem of learning {\sl stochastic decision trees}, a generalization of standard deterministic decision trees that allows for stochastic nodes.  This generalization broadens the expressive power of decision trees, enabling them to represent not just deterministic functions but also stochastic functions.  \Cref{fig:SDT} depicts a stochastic decision tree with two stochastic nodes, labeled `$\$$', one that branches on the outcome of a $\mathrm{Bernoulli}(0.8)$ random variable, and the other on the outcome of a $\mathrm{Bernoulli}(0.3)$ random variable.

\begin{figure}[htb]
\forestset{
  triangle/.style={
    node format={
      \noexpand\node [
      draw,
      shape=regular polygon,
      regular polygon sides=3,
      inner sep=0pt,
      outer sep=0pt,
      \forestoption{node options},
      anchor=\forestoption{anchor}
      ]
      (\forestoption{name}) {\foresteoption{content format}};
    },
    child anchor=parent,
  }
 }
\begin{align*}
    \hspace{3mm}{\small
    \!\begin{gathered}
    \begin{forest}
    for tree={
        grow=south,
        circle, draw, minimum size=11mm, l sep = 8mm, s sep = 3mm, inner sep = 1mm
    }
    [$x_1$
        [$x_2$, edge label = {node [midway, fill=white] {$0$} }
            [$1$, edge label = {node [midway, fill=white] {$0$} }]
            [$\$$, edge label = {node [midway, fill=white] {$1$} }
                [$1$, edge label = {node [midway, fill=white] {$0.2$} }]
                [$0$, edge label = {node [midway, fill=white] {$0.8$} }]
            ]        ]
        [$\$$, edge label = {node [midway, fill=white] {$1$} }
            [$x_3$, edge label = {node [midway, fill=white] {$0.7$} }
                [$0$, edge label = {node [midway, fill=white] {$0$} }]
                [$1$, edge label = {node [midway, fill=white] {$1$} }]
            ]
            [$1$, triangle, edge label = {node [midway, fill=white] {$0.3$} }]
        ]
    ]
    \end{forest}
    \end{gathered}}
\end{align*}
\caption{A stochastic decision tree with two stochastic nodes.}
\label{fig:SDT} 
\end{figure}
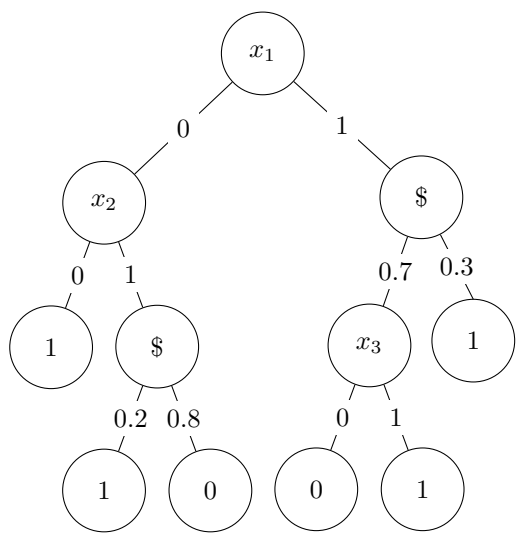


Many real-world learning scenarios are inherently  stochastic in nature, and relatedly, much of current research in learning theory focuses on the ``probabilistic concept'' generalization~\cite{KS94} of the standard PAC model of learning deterministic concepts (e.g.~see~\cite{GKM18,GK19,GGJKK20,GGK20} for an ongoing line of work on learning neural networks in the probabilistic concept model).  As discussed in~\cite{KS94}, probabilistic concepts can also be viewed as latent variable models, where the uncertainty concerning latent variables is modeled as apparent probabilistic behavior.

Stochastic decision trees are a simple and natural way to represent stochastic functions.  Despite compelling theoretical and practical motivations, there has thus far been considerably less attention on the problem learning stochastic decision trees as compared to deterministic decision trees.  Many basic questions remain open; for example: 
\begin{itemize}
\item[$\circ$]  Is there an algorithm for {\sl properly} learning stochastic decision trees, one that returns a decision tree hypothesis? 
\item[$\circ$] Is there an algorithm for learning stochastic decision trees that is resilient to {\sl adversarial noise}? 
\end{itemize} 

These questions have been intensively studied in the case of deterministic decision trees, and the algorithms and techniques developed to answer them (e.g.~\cite{EH89,KKMS08,GKK08}) have become foundational results in learning theory.   A broad goal of our work is to help bring the state of our understanding of learning stochastic decision trees into closer alignment with that of deterministic decision trees. 

\subsection{Our results} 
We give new algorithms for learning stochastic decision trees under the uniform distribution.  En route to our main result, we give the first algorithm for {\sl properly} learning stochastic decision trees---our algorithm in fact returns a deterministic decision tree hypothesis: 

\begin{theorem}[Properly learning stochastic decision trees] 
\label{thm:noiseless} 
There is an algorithm $\mathcal{A}$ with the following guarantee.  For all $\eps\in (0,1)$ and $s \in \N$, given access to labeled samples $(\bx,\bT(\bx))$ where $\bT : \zo^n \to \zo$ is a size-$s$ stochastic decision tree and $\bx$ is uniform random, $\mathcal{A}$ runs in $n^{O(\log(s/\eps)/\eps^2)}$ time and with high probability outputs a deterministic decision tree $h$ such that $\Pr[h(\bx)\ne \bT(\bx)] \le \opt + \eps$, where $\opt$ denotes the Bayes optimal error for $\bT$. \end{theorem}


\Cref{thm:noiseless} is a special case of our main result, which gives a generalization of the algorithm $\mathcal{A}$ of~\Cref{thm:noiseless}  that is optimally resilient to adversarial noise.  

\begin{definition}[$\eta$-corrupted samples; ``nasty noise''~\cite{BEK02}] 
Let $\boldf : \zo^n \to \zo$ be a stochastic function.  We say that $\mathcal{S}$ is an {\sl $\eta$-corrupted set of uniform random samples labeled by $\boldf$} if it is formed in the following fashion: draw a set of labeled samples $(\bx,\boldf(\bx))$ where $\bx$ is uniform random, and modify any $\eta$ fraction to form $\mathcal{S}$. 
\end{definition} 

We allow for corruptions of both the example (i.e.~changing $\bx$ to a different $\bx'$) and its label (i.e.~flipping $\boldf(\bx)$), and note that the adversarial choice of which $\eta$ fraction of samples to corrupt can be adaptive, depending arbitrarily on the original uncorrupted set of samples.   This is regarded as the most challenging noise model for classification problems; weaker noise models include random classification noise, Massart noise, and agnostic noise. 

 Our main result is as follows: 

\begin{theorem}[Our main result: Properly learning stochastic decision trees in the presence of adversarial noise] 
\label{thm:main}
There is an algorithm $\mathcal{A}$ with the following guarantee.  For all $\eps, \eta \in (0,1)$ and $s \in \N$, given access to a sufficiently large $\eta$-corrupted set $\mathcal{S}$ of uniform random samples labeled by a size-$s$ stochastic decision tree $\bT : \zo^n \to \zo$, $\mathcal{A}$ runs in $n^{O(\log(s/\eps)/\eps^2)}$ time and with high probability outputs a decision tree hypothesis $h$ such that $\Pr[h(\bx)\ne \bT(\bx)] \le \opt + 2\eta + \eps$, where $\opt$ denotes the Bayes optimal error for $\bT$.
\end{theorem} 

An error of $\opt + 2\eta$ is the information-theoretic  minimum (see e.g.~\cite{BEK02}).  Prior to our work there were (improper) algorithms that achieved either $\opt + O(\sqrt{\eta}) + \eps$ or $2\hspace{.1em}\opt + 2\eta + \eps$, the low-degree algorithm of~\cite{LMN93} and the $L_1$ polynomial regression algorithm of~\cite{KKMS08} respectively, but not the information-theoretically optimal $\opt + 2\eta + \eps$.  This was the case even for weaker noise models such as label-only noise (i.e.~agnostic noise~\cite{Hau92,KSS94}).  In fact, the low-degree and $L_1$ polynomial regression algorithms are, in general, only known to be resilient to noise in the labels.

As our final contribution, we show that when applied in the context of decision tree learning, these algorithms are in fact resilient to noise in both the examples and their labels: 

\begin{theorem}[Noise-tolerant properties of the low-degree algorithm and $L_1$ polynomial regression]
\label{thm:polynomial regression}
For all $\eps, \eta \in (0,1)$ and $s \in \N$, given access to a sufficiently large $\eta$-corrupted set $\mathcal{S}$ of uniform random samples labeled by a size-$s$ stochastic decision tree $\bT : \zo^n \to \zo$, 
\begin{itemize}
    \item[$\circ$] the low-degree algorithm runs in time $n^{O(\log(s/\eps))}$ and with high probability outputs a hypothesis $h$ satisfying $\Pr[h(\bx)\ne \bT(\bx)]\le \opt + O(\sqrt{\eta}) + \eps$. 
    \item[$\circ$] the $L_1$ polynomial regression algorithm runs in time $n^{O(\log(s/\eps))}$ and with high probability outputs a stochastic hypothesis $\bh$ satisfying $\Pr[\bh(\bx)\ne\bT(\bx)]\le 2\hspace{.1em}\opt + 2\eta + \eps$.
\end{itemize}
\end{theorem}

\subsubsection{Summary and comparison with existing algorithms} The low-degree algorithm of Linial, Mansour, and Nisan~\cite{LMN93} and a recent algorithm of Chen and Moitra~\cite{CM19} for learning mixtures of subcubes can both be used to learn stochastic decision trees as a special case of their main results.  The algorithm of~\cite{LMN93} runs in time $n^{O(\log(s/\eps))}$, whereas the algorithm of~\cite{CM19} runs in time $O_s(1) \cdot n^{O(\log s)}\cdot \poly(1/\eps)$.  However, neither of these algorithms returns a decision tree hypothesis, and hence both are improper when applied in this context.  The classic algorithm of Ehrenfeucht and Haussler~\cite{EH89,Blu92} properly learns deterministic decision trees in time $n^{O(\log s)}\cdot \poly(1/\eps)$.  However, being an Occam algorithm, its analysis seems fundamentally unable to accommodate stochasticity of the target concept. 

\Cref{prior:Gaussian} summarizes our contributions and places them in the context of prior work. 

\vspace{5pt}
\begin{table}[h!]
\begin{adjustwidth}{-1in}{-1in}
  \captionsetup{width=.72\linewidth}
\renewcommand{\arraystretch}{1.8}
\centering
\begin{tabular}{|c|c|c|c|}
\hline
~~Reference~~   & Technique & Running time & Error guarantee \\ [.2em] \hline \hline
\multirow{2}{*}{\cite{LMN93}} & \multirow{2}{*}{Low-degree algorithm}  & \multirow{2}{*}{$n^{O(\log(s/\eps))}$} &  $\opt + O(\sqrt{\eta}) + \eps$ \\ [-.8em]
 &  & & (\small{{\bf This work}}) \\ [.2em] \hline 
 \multirow{2}{*}{\cite{KKMS08}} & \multirow{2}{*}{$L_1$ polynomial regression} & \multirow{2}{*}{$n^{O(\log(s/\eps))}$}  & $2\hspace{.1em}\opt + 2\eta + \eps$ \\ [-.8em] 
  & & & (\small{{\bf This work}}) \\ [.2em] \hline 
 \multirow{2}{*}{\cite{CM19}} & \multirow{2}{*}{Learning mixtures of subcubes} & \multirow{2}{*}{~~$O_s(1) \cdot n^{O(\log s)} \cdot \poly(\lfrac1{\eps})$~~} & $\opt + \eps$ \\ [-.8em]  
&  & & ~~\small{Noiseless setting ($\eta = 0$)}~~  \\   [.2em] \hline \hline 
\multirow{2}{*}{\bf This work}  & \small{Approximation by stochastic-leaf DTs;} &  \multirow{2}{*}{$n^{O(\log(s/\eps)/\eps^2)}$} & \multirow{2}{*}{$\opt + 2\eta + \eps$} \\ [-.9em] 
& \small{~Noise-tolerant learning of stochastic-leaf DTs~} & & \\ [.2em] \hline 
\end{tabular}\vspace{5pt} 
\end{adjustwidth}
\caption{Performance guarantees of our algorithm and existing algorithms for learning stochastic decision trees in the presence of adversarial noise.  Among these algorithms, ours is the only one that returns a decision tree hypothesis.  
Prior to our work, the error guarantees for the low-degree algorithm and $L_1$ polynomial regression were only known for label noise; we show in the context of decision tree learning, these guarantees can be strengthened to allow for noise in both the examples and labels.}
\vspace{-15pt}
\label{prior:Gaussian}\end{table}

\subsection{Our techniques} 

Our approach to~\Cref{thm:noiseless,thm:main} is simple and has two main conceptual parts: a structural lemma concerning stochastic decision trees and a noise-tolerant algorithm for learning a special type of stochastic decision tree.  

\begin{itemize} 
\item[$\circ$] {\sl Structural lemma:}  We show that every size-$s$ stochastic decision tree can be $\eps$-approximated by a ``stochastic-leaf decision tree'' of size $s^{O(1/\eps^2)}$.  A stochastic-leaf decision tree is a very specific type of stochastic decision tree, one whose stochastic nodes only occur at its leaves.
\end{itemize}
This lemma reduces the task of learning stochastic decision trees to that of learning stochastic-leaf decision trees, with a catch: due to the approximation error incurred, the algorithm for learning stochastic-leaf decision trees has to be {\sl noise-tolerant}. 
\begin{itemize} 
\item[$\circ$] {\sl Noise-tolerant learning stochastic-leaf decision trees:}  Mehta and Raghavan~\cite{MR02}  gave an algorithm for properly learning deterministic decision trees in the noiseless setting.  We show that their algorithm can be generalized to handle stochastic-leaf decision trees, and furthermore, we show that our generalization is optimally resilient to adversarial noise.  This stands in contrast to the algorithm of Ehrenfeucht and Haussler~\cite{EH89}, which as mentioned above seems fundamentally unable to accommodate either stochasticity or noise. 


\end{itemize} 

We are hopeful that each of these two parts will see further utility in problems involving stochastic decision trees, beyond the learning-theoretic setting that is the focus of this work. 

As for~\Cref{thm:polynomial regression}, the low-degree algorithm and $L_1$ polynomial regression are versatile and powerful ``meta-algorithms'' in learning, but they are not generally known to handle the challenging nasty noise. Our key observation here is that the mean functions of stochastic decision trees are well-approximated by low-degree polynomials {\sl with bounded outputs}. We then show that when run on such polynomials, the low-degree algorithm and $L_1$ polynomial regression are in fact resilient to nasty noise.  Given the broad applicability of both algorithms, we are similarly hopeful that this fact will be of independent interest beyond decision trees. 

\subsection{Preliminaries} 
\label{sec:prelims} 
Let $\boldf : \zo^n \to \zo$ be a  stochastic function.  We associate $\boldf$ with its {\sl  mean function} $\mu_\boldf : \zo^n \to [0,1]$, 
$ \mu_\boldf(x) \coloneqq \Prx_{\boldf} [\boldf(x) = 1].$ 
The {\sl Bayes optimal classifier} for $\boldf$ is the (deterministic) function $x \mapsto \round(\mu_\boldf(x)),$ where $\round(t) \coloneqq \Ind[t \ge \frac1{2}]$.
Given two stochastic functions $\boldf,\bh : \zo^n \to \zo$, we define 
\[ \error_\boldf(\bh) \coloneqq \Ex_{\bx}\Big[ \Prx_{\boldf,\bh}[\boldf(\bx) \ne \bh(\bx)]\Big], \] 
where here and throughout this paper, $\bx$ denotes a uniform random input from $\zo^n$.  We define $\opt_\boldf \coloneqq \error_\boldf(\round(\mu_\boldf))$, and  
when $\boldf$ is clear from context, we simply write $\opt$. 
\begin{fact}[Bayes optimal classifier minimizes classification error]
For all stochastic functions $\boldf, \bh : \zo^n \to \zo$, we have $\error_\boldf(\bh) \ge \opt_\boldf.$
\end{fact}

\begin{fact}[$L_1$-error and Bayes optimality] 
\label{fact:L1-to-Bayes} 
Let $\boldf : \zo^n \to \zo$ be a stochastic function. For any $h : \zo^n \to [0,1]$,
\[ \Pr[\round(h(\bx)) \ne \boldf(\bx)] \le \opt_\boldf + 2\E[|\mu_\boldf(\bx)-h(\bx)|]. \] 
\end{fact} 
\Cref{fact:L1-to-Bayes} states that if we have a function close to $\mu_{\boldf}$, we can convert it to a classifier with error close to $\opt_{\boldf}$.
\begin{proof}
    We need to upper bound $\error_{\boldf}(\round \circ h)  - \BayesOpt_{\boldf}$ at $2\E[|\mu_\boldf(\bx)-h(\bx)|] $. We rewrite that quantity as
    \begin{align*}
        \error_{\boldf}(\round \circ h)  - \BayesOpt_{\boldf} &= \Prx_{\bx \sim \zo^n}\left[\boldf(\bx) \neq \round(h(\bx)) \right] - \Prx_{\bx \sim \zo^n}\left[
        \boldf(x) \neq \round(\mu_\boldf(\bx)) \right] \\
        &= \Ex_{\bx \sim \zo^n}\left[|\,\mu_{\boldf}(\bx) -  \round(h(\bx)) \,|  - 
        |\,\mu_{\boldf}(\bx) - \round(\mu_\boldf(\bx))\,| \right] \\
        &= \Ex_{\bx \sim \zo^n}\left[\Ind(\round(h(\bx)) ) \neq \round(\mu_\boldf(\bx))) \cdot 2 \cdot 
        \left|\,\mu_{\boldf}(\bx) - \lfrac{1}{2}\,\right| \right]
    \end{align*}
     It is only possible that $\round(h(x)) \neq \round(\mu_\boldf(x))$ if $|\,f(x) -\mu_\boldf(x)\,| \geq  |\,\mu_\boldf(x) - \frac{1}{2}\,|$. Therefore,
    \begin{align*}
        \error_{\boldf}(\round \circ h)  - \BayesOpt_{\boldf} &\leq \Ex_{\bx \sim \zo^n}\left[\Ind(|\,f(x) -\mu_\boldf(x)\,| \geq  |\,\mu_\boldf(x) - \frac{1}{2}\,|) \cdot 2 \cdot 
        \left|\,\mu_{\boldf}(\bx) - \lfrac{1}{2}\,\right| \right] \\
        &\leq 2 \Ex_{\bx \sim \zo^n}\left[|\,f(x) -\mu_\boldf(x)\,| \right]. \qedhere
    \end{align*}
\end{proof}

\section{Approximating stochastic DTs with stochastic-leaf DTs} 
\label{sec:stochastic leaf}
\begin{definition}[Stochastic-leaf DT]
A stochastic-leaf DT is a stochastic DT for which all stochastic nodes have only leaves as their children. 
\end{definition}

\begin{lemma}[Approximating stochastic DTs with stochastic-leaf DTs]
\label{lem:approx-stochastic-DT}
Let $\bT$ be a size-$s$ stochastic DT.  For every $\eps \in (0,\frac1{2})$, there is a size-$S${\emph{ stochastic-leaf}} DT $\overline{\bT}$ such that $S \le s^{O(1/\eps^2)}$ and 
$ \Ex_{\bx}[|\mu_{\bT}(\bx)-\mu_{\overline{\bT}}(\bx)|] \le \eps.$\end{lemma}


\begin{proof}
Let $m$ denote the number of stochastic transitions in $\bT$. For a fixed $r \in \zo^m$, let $\bT(x, r)$ be the value of $\bT$ evaluated on $x$ with stochastic transitions determined by $r$. Suppose we pick random strings $\br_1,\ldots, \br_c \sim \zo^{m}$ independently and uniformly at random.  For each $x \in \zo^n$, consider the following random variable:
    \begin{align*}
        \mathbf{est}(x) &\coloneqq \Ex_{\bi \in [c]}[\bT(x, \br_{\bi})].
    \end{align*}
    Note that 
    \begin{align*} \E_{\br_1, \ldots, \br_c \in \zo^m}[\mathbf{est}(x)] &= \mu_\bT(x) = \Ex_{\br \sim\zo^m}[\bT(x,\br)] \\
    \Var[\mathbf{est}(x)] &= \lfrac{1}{c}\cdot \Varx_{\br\sim\zo^m}[\bT(x,\br)], 
    \end{align*} 
        where in both cases above, $\br \sim \zo^m$ on the RHS denotes $\br$ chosen uniformly at random from $\zo^m$. 
    Since $\bT$ is $\zo$-valued, it has variance at most $\frac{1}{4}$. Hence, the variance of $\mathbf{est}(x)$ is at most $\frac{1}{4c}$. If we take $c =1/\varepsilon^2$, the following holds for any $x \in \zo^n$:
\[        \Ex_{\br_1, \ldots, \br_c \sim \zo^m}\Big[\big(\mathbf{est}(x) - \mu_\bT(x)\big)^2\Big] \leq \frac{\varepsilon^2}{4}, \quad \text{and therefore}\ 
        \Ex_{\br_1, \ldots, \br_c \sim \zo^m}[|\mathbf{est}(x) - \mu_\bT(x)|] \leq \frac{\varepsilon}{2}.
\]     
Averaging over $\bx \sim \zo^n$ and swapping expectations, we get: 
    \begin{align*}
        \Ex_{\br_1, \ldots, \br_c \sim \zo^m}\bigg[\Ex_{\bx \sim \zo^n}\Big[|\mathbf{est}(x) - \mu_\bT(x)|\Big] \bigg] \leq \frac{\varepsilon}{2}.
    \end{align*}
Therefore, there must exist outcomes $r_1^{\star}, \ldots, r_c^{\star} \in \zo^m$ of $\br_1,\ldots,\br_c$ such that 
    \begin{align}
    \label{eq: existence small error}
        \Ex_{\bx \sim \zo^n}\bigg[\big|\Ex_{\bi \in [c]}[\bT(\bx, r_{\bi^\star})] - \mu_\bT(\bx)\big |\bigg] \leq \frac{\varepsilon}{2}.
    \end{align}
    For each $i \in [c]$, we define a size-$s$ DT by fixing the stochastic nodes of $\bT$ according to $r_i^\star \in \zo^m$. We define our stochastic-leaf DT $\overline{\bT}$ by stacking these $c$ many size-$s$ DTs on top of one another: for each $i < n$, we replace each leaf of the $i_{th}$ DT with a copy of the $(i+1)_{th}$ DT. Then for each leaf $\ell$ of this stacked tree, let $x_{\ell}$ be an input that is consistent with the root-to-$\ell$ path in $\overline{\bT}$. We replace $\ell$ with a stochastic node which transitions to a 1-leaf with probability $p_\ell := \Ex_{\bi \in [c]}[\bT(x_\ell, r_{\bi}^\star)]$, and to a 0-leaf with probability $1 - p_\ell$.  Note that for each $i \in [c]$, the tree $\bT(\cdot, r_i^\star)$ gives the same classification for all inputs reaching leaf $\ell$ of $\overline{\bT}$, so $p_\ell$ does not depend on the choice of $x_\ell$.
    
    $\overline{\bT}$ is a stochastic-leaf DT that computes $x \mapsto \Ex_{\bi \in [c]}[\bT(\bx, r_{\bi}^\star)]$, which by \Cref{eq: existence small error}, has sufficiently small error. Since this DT has size $s^c = s^{O(1/\eps^2)}$, the proof of~\Cref{lem:approx-stochastic-DT} is complete.
\end{proof}

\section{A simple backtracking algorithm for finding the optimal small-depth tree}

The algorithmic core of \Cref{thm:noiseless,thm:main} is a  recursive backtracking procedure {\sc Find} shown in \Cref{fig:find}, which takes a labeled set of samples $X$ and finds a depth-$d$ decision tree that achieves minimal classification error. This algorithm is inspired by and simplifies the {\sc Find} algorithm given by Mehta and Raghavan \cite{MR02} for building a minimum-error decision tree from any ``sat-countable representation'' of a function.

\begin{figure}[ht!]
  \captionsetup{width=.9\linewidth}
\begin{tcolorbox}[colback = white,arc=1mm, boxrule=0.25mm]
    \textsc{Find}$(X,d)$:
    \begin{enumerate}[align=left]
        \item[\textbf{Input:}]  Set $X$ of labeled examples $(x,y)$ and depth budget $d$.
        \item[\textbf{Output:}]  A depth-$d$ DT $T^\star$ that minimizes $\Prx_{(\bx,\by) \sim X}[T^{\star}(\bx) \ne \by]$ among all depth-$d$ DTs.
    \end{enumerate}
    \begin{enumerate}
        \item If $d = 0$, return the constant $c\in \{0,1\}$ that minimizes $\Prx_{(\bx,\by) \sim X}[c \ne \by]$.
        \item For every $i \in [n]$, let $T_i$ be the DT defined as follows: 
        \begin{itemize}
        \item[$\circ$] $T_i$ queries $x_i$ at the root; 
        \item[$\circ$] Has $\textsc{Find}(X_{x_i = 0}, d-1)$ as its left subtree;
        \item[$\circ$] Has $\textsc{Find}(X_{x_i = 1}, d-1)$ as its right subtree.
        \end{itemize}
Here $X_{x_i = b}$ denotes the subset of $X$ containing only examples where $x_i$ is set to $b$. 
        \item Return the tree $T_{i^\star}$ that minimizes $\Prx_{(\bx,\by) \in X}[T_{i}(\bx) \ne \by]$ among all $i \in [n]$.
    \end{enumerate} 
\end{tcolorbox} 
\caption{A recursive backtracking algorithm for finding a depth-$d$ DT of minimal classification error.} 
\label{fig:find}
\end{figure} 

\begin{lemma}[Correctness of \textsc{Find}]
    \label{lemma:find correctness}
    Consider any sample set $X$ of labeled examples $(x,y)$ and depth budget $d$. The algorithm \textsc{Find}$(X, d)$ of~\Cref{fig:find} returns a depth-$d$ DT $T^\star$ that minimizes $\Prx_{(\bx,\by) \sim X}[T^\star(\bx) \ne \by]$ among all depth-$d$ DTs.
\end{lemma}

\begin{proof}
    We proceed by induction on $d$. If $d = 0$, then $\textsc{Find}$ returns at Step 1 and is clearly correct.
    For the inductive step, suppose that $d \geq 1$. For any $i\in [n]$, we first claim that the tree $T_i$ defined in Step 2 is a depth $d$ DT that minimizes classification error with respect to $X$ among those that query $x_i$ at the root. Let $(T_i)_{\mathrm{left}}$ and $(T_i)_{\mathrm{right}}$ be its left and right subtrees respectively. By the inductive hypothesis, the left and right subtrees $(T_i)_{\mathrm{left}}$ and $(T_i)_{\mathrm{right}}$ are depth $d-1$ DTs that minimize error with respect to $X_{x_i=0}$ and $X_{x_i=1}$ respectively. Hence, $T_i$ is a depth $d$ DT that achieves minimal error with respect to $X$ among those that query $x_i$ at the root. 
    
    Since $\textsc{Find}$ returns the $T_{i^\star}$ that minimizes $\Prx_{(\bx, \by) \sim X}[T_i(\bx) \ne \by]$ among all $i \in [n]$ in Step 3, and each $T_i$ is a minimal-error depth-$d$ DT among those that query $x_i$ at the root, we conclude that $\textsc{Find}$ returns a tree of minimal error with respect to $X$.
\end{proof}

\begin{lemma}[Efficiency of {\sc Find}]
\label{lemma:find efficiency} 
    Consider any sample set $X$ of labeled examples and depth budget $d$. The algorithm $\textsc{Find}(X, d)$ of~\Cref{fig:find} takes time $n^{O(d)} \cdot O(|X|)$. 
\end{lemma}

\begin{proof}
Let $T(d)$ denote the running time of {\sc Find} when run with depth budget $d$. If $d = 0$ then the algorithm only executes Step 1, which can be done in $O(|X|)$ time by computing $\round(\Ex_{(\bx, \by) \sim X}[\by])$. 
    
Next we consider the case of $d \ge 1$.  In step 2, $\textsc{Find}$ recurses $2n$ times, each with $d$ decremented by one. Each time it also partitions $X$ into $X_{x_i = 0}$ and $X_{x_i = 1}$. All of these recursive calls and partitioning takes total time $2n \cdot T(d-1) + n |X|.$  In step 3, $\textsc{Find}$ must compute $\Prx_{(\bx, \by) \sim X}[T_i(\bx) \ne \by]$ for up to $n$ different coordinates $i$, where each $T_i$ has depth at most $d$. This takes time $n \cdot d \cdot |X|$.  We therefore have the recurrence relation: 
\[ T(d) \le 2n \cdot T(d -1) + O(nd|X|).\] 
Solving this recurrence relation gives us the bound $T(d) \le (2n)^d \cdot O(nd|X|)$, which is $\le n^{O(d)} \cdot O(|X|)$ as desired.
\end{proof}

\violet{ 

}
\section{Learning stochastic DTs: proofs of~\Cref{thm:noiseless,thm:main}}
\subsection{Proof of \Cref{thm:noiseless}}
We recall \Cref{thm:noiseless}, this time including the confidence parameter $\delta$. 

\begin{restatement}{1}[Properly learning stochastic decision trees] 
There is an algorithm $\mathcal{A}$ with the following guarantee.  For all $\eps\in (0,1)$ and $s \in \N$, given access to labeled samples $(\bx,\bT(\bx))$ where $\bT : \zo^n \to \zo$ is a size-$s$ stochastic decision tree and $\bx$ is uniform random, $\mathcal{A}$ runs in $n^{O(\log(s/\eps)/\eps^2)} \cdot \poly(\log(1/\delta))$ time and with probability $1 - \delta$ outputs a deterministic decision tree $h$ such that $\Pr[h(\bx)\ne \bT(\bx)] \le \opt + \eps$, where $\opt$ denotes the Bayes optimal error for $\bT$. \end{restatement} 

Let $\bT$ be a size-$s$ stochastic decision tree.  By~\Cref{lem:approx-stochastic-DT},   there is a stochastic-leaf decision tree $\overline{\bT}$ of size $S \le s^{O(1/\eps^2)}$ such that $\Ex_{\bx}[|\mu_{\bT}(\bx)-\mu_{\overline{\bT}}(\bx)|] \le \eps$. Consider the Bayes optimal classifier $x \mapsto \round(\mu_{\overline{\bT}}(x))$ for $\overline{\bT}$.  Since $\overline{\bT}$ is a {\sl stochastic-leaf} decision tree, we have that this function is computed by a size-$S$ (deterministic) decision tree $T^\star$: to obtain $T^\star$ from $\overline{\bT}$, simply replace every stochastic node in $\overline{\bT}$, all of which occur at the leaves of $\overline{\bT}$, with a $1$-leaf if it branches on $\mathrm{Bernoulli}(p)$ where $p \ge \frac1{2}$, and a $0$-leaf otherwise. Applying~\Cref{fact:L1-to-Bayes}, we get that 
\[ \Prx_{\bx,\bT}[T^\star(\bx) \ne \bT(\bx)] \le \opt_{\bT}+ 2\eps.  \] 
Next, consider the decision tree $T^{\star}_\trunc$ obtained by truncating $T^\star$ to depth $\log(S/\eps)$ (and replacing all truncated branches with a leaf with an arbitrary value, say a $1$-leaf). $T^{\star}_\trunc$ and $T^\star$ can only differ on inputs that reach a leaf in $T^{\star}$ of depth at least $\log(S/\eps)$, and there are at most $S$ such leaves. Therefore, 
\[ \Prx_{\bx}[T^\star_\trunc(\bx) \ne T^\star(\bx)] \le 2^{-\log(S/\eps)} \cdot S = \eps.   \] 

Note that the depth of $T^\star_\trunc$ is $\le \log(s/\eps)/\eps^2$. We have shown the following corollary of~\Cref{lem:approx-stochastic-DT}:

\begin{corollary}[Approximating stochastic DTs with deterministic ones]
\label{cor:approx-stochastic-DT}  Let $\bT : \zo^n \to \zo$ be a size-$s$ stochastic DT.  For every $\eps \in (0,\frac1{2})$, there is a deterministic DT $T^{\star}_{\mathrm{trunc}} : \zo^n \to \zo$ such that 
\begin{enumerate}
    \item  $\mathrm{depth}(T^{\star}_{\mathrm{trunc}}) \le \log(S/\eps) \leq \log(s/\eps)/\eps^2$ and 
    \item $\ds  \Prx_{\bx,\bT}[T^{\star}_{\mathrm{trunc}}(\bx) \ne \bT(\bx)] \le \opt_{\bT} +  3\eps.$
\end{enumerate}
\end{corollary}

To show that {\sc Find} returns a tree of small error with respect to $\bT$, we need the following generalization bound from \cite{MR02}: 
\begin{lemma}[Generalization]
\label{lem:generalization}
Let $\bT$ be a stochastic tree of size $s$. For $S = s^{O(1/\eps^2)}$ and a sample size of 
\begin{align*}
        m \coloneqq \poly\left(n^{\log(S/\eps)}, \frac{1}{\eps}, \log\left(\frac{1}{\delta}\right) \right),
    \end{align*}
let $X$ be a dataset of $m$ i.i.d points of the form $(\bx, \bT(\bx))$. Then $\textsc{Find}(X, \log(S/\eps))$ outputs $T^\star$ such that 
\[\Prx_{\text{draw of }\bX}\big [\Prx_{\bx \sim \{0,1\}^n} [T^\star(\bx) \ne \bT(\bx)] \le \opt_{\bT} + 3\eps \big ] \ge 1 - \delta.\]
\end{lemma}

\begin{proof}
The proof is given in the proof of Theorem 2 in \cite{MR02}. \Cref{lemma:find correctness} gives us that {\sc Find} outputs a tree of minimal error with respect to $X$. They apply Chernoff bounds to bound the probability that a fixed tree $T'$ of depth $\log(S/\eps)$ and error $> \opt_\bT + 3\eps$ with respect to $\bT$ has smaller error with respect to $X$ than $T^{\star}_{\mathrm{trunc}}$ as described in \Cref{cor:approx-stochastic-DT}. More specifically, the probability over draws of $X$ that $\Prx_{(\bx, \by) \sim X} [T^{\star}_{\mathrm{trunc}}(\bx) \ne \by] > \opt_\bT + 3\eps$ or $\Prx_{(\bx, \by) \sim X} [T'(\bx) \ne \by] \le \opt_\bT + 3\eps$ is exponentially small in $|X|$. This is a bound on the probability that {\sc Find} outputs a particular tree of error greater than $\opt_{\bT} + 3\eps$; the lemma follows from a union bound over all trees of depth at most $\log(S/\eps)$.
\end{proof}

\Cref{lemma:find efficiency} gives us that $\textsc{Find}(X, \log(S/\eps))$ runs in time $n^{O(\log(S/\eps))} \cdot O(|X|)$ = $n^{O(\log(s/\eps)/\eps^2)} \cdot O(|X|)$. For confidence parameter $\delta$, $|X|$ is polynomial in $n^{\log(S/\eps)}$, $\log(1/\eps)$, and $\log(1/\delta)$. Thus, the total runtime of {\sc Find} is $n^{O(\log(s/\eps)/\eps^2)} \cdot \poly\log(1/\delta))$. The desired result holds by renaming $\eps'= \eps/3$. \qed

\subsection{Proof of \Cref{thm:main}}
We recall \Cref{thm:main}, this time including the confidence parameter $\delta$. 
\begin{restatement}{3}[Our main result] 
There is an algorithm $\mathcal{A}$ with the following guarantee.  For all $\eps, \eta \in (0,1)$ and $s \in \N$, given access to a sufficiently large $\eta$-corrupted set $\mathcal{S}$ of uniform random samples labeled by a size-$s$ stochastic decision tree $\bT : \zo^n \to \zo$, $\mathcal{A}$ runs in $n^{O(\log(s/\eps)/\eps^2)} \cdot \poly(\log(1/\delta))$ time and with probability $1 - \delta$ outputs a decision tree hypothesis $h$ such that $\Pr[h(\bx)\ne \bT(\bx)] \le \opt + 2\eta + \eps$, where $\opt$ denotes the Bayes optimal error for $\bT$.
\end{restatement} 

The proof requires the following fact. 
\begin{fact}[Error from sample corruption]
    \label{fact:bound on corruption}
    For any bounded function $p: \zo^n \to [0,1]$ and sample $\mathcal{S}^\circ$ of points $(x_1,y_1), \ldots, (x_m, y_m)$ with $0 \leq y_i \leq 1$. Let $\mathcal{S}$ be a corrupted sample formed by picking an arbitrary $\eta$-fraction of points $\mathcal{S}^\circ$ and replacing each with an arbitrary (also bounded) point. Then for any $\mathrm{err}: [0,1] \times [0,1] \to [0,1]$
    \begin{align*}
        \left|\Ex_{(\bx, \by) \sim \mathcal{S}^\circ}\left[\mathrm{err}(p(\bx),\by) \right] - \Ex_{(\bx, \by) \sim \mathcal{S}}\left[\mathrm{err}(p(\bx),\by)\right] \right| < \eta. 
    \end{align*}
\end{fact}

Recall $T^{\star}_{\mathrm{trunc}}$ as described in \Cref{cor:approx-stochastic-DT}, which has error $\le \opt_\bT + O(\eps)$ with respect to $\bT$. Let $\mathcal{S}^\circ$ be the uncorrupted set of examples of $\bT$, and $\mathcal{S}$ be an $\eta$-corruption of $\mathcal{S}^{\circ}$. Then with probability $1-\delta$ over draws of $\mathcal{S}^\circ$,
\begin{align*}
    \Prx_{(\bx, \by) \sim \mathcal{S}^\circ} [T^{\star}_{\mathrm{trunc}}(\bx) \ne \by] &\le \opt_\bT + 3\eps \tag{\Cref{lem:generalization}}\\
    \Prx_{(\bx, \by) \sim \mathcal{S}} [T^{\star}_{\mathrm{trunc}}(\bx) \ne \by] &\le \opt_\bT + \eta + 3\eps. \tag{\Cref{fact:bound on corruption}}
\end{align*}
Let $T^\star$ be the output of $\textsc{Find}(\mathcal{S}, \log(S/\eps))$. Then,
\begin{align*}
    \Prx_{(\bx, \by) \sim \mathcal{S}} [T^{\star}(\bx) \ne \by] &\le \opt_\bT + \eta + 3\eps \tag{\Cref{lemma:find correctness}}\\
    \Prx_{(\bx, \by) \sim \mathcal{S}^\circ} [T^{\star}(\bx) \ne \by] &\le \opt_\bT + 2\eta + 3\eps. \tag{\Cref{fact:bound on corruption}} \\
    \Prx_{\text{draw of }\mathcal{S}^\circ}\big[\Prx_{x \sim \{0,1\}^n}[T^\star(x) \ne \bT(x)] &\le \opt_\bT + 2\eta + 3\eps \big] > 1 - \delta \tag{\Cref{lem:generalization}} 
\end{align*}
The desired result holds by renaming $\eps'= \eps/3$. \qed


\section{Noise-tolerant properties of $L_1$ and $L_2$ regression}
\label{sec:polynomial regression}

In this section, we prove \Cref{thm:polynomial regression}, showing that the low-degree algorithm of~\cite{LMN93} (also known as $L_2$ regression) and $L_1$ regression algorithm of~\cite{KKMS08} both learn stochastic-leaf DTs with adversarial corruption, albeit with worse parameters than our method. Throughout this section, we use the following function.
\begin{definition}[The $\trunc$ function]
    The function, $\trunc:\R \to [0,1]$, is defined as
    \begin{align*}
        \trunc(x) = 
        \begin{cases}
            0 &\text{if $x < 0$} \\
            1 &\text{if $x > 1$} \\
            x &\text{otherwise.}
        \end{cases}
    \end{align*}
\end{definition}

The basis of the results in this section is \Cref{prop:mu-low-degree-bounded}, that if $\bT$ is a size-$s$ stochastic DT, there is a degree $\log(s/\eps)$ bounded polynomial $p: \zo^n \to [0,1]$ which is $\eps$ close to $\mu_\bT$: 

\begin{proposition}[$\mu_\bT$ is morally low degree]
    \label{prop:mu-low-degree-bounded}
    Let $\bT$ be a size-$s$ stochastic DT.  There is a polynomial $p : \zo^n \to [0,1]$ such that 
    $ \Prx_{\bx}[p(\bx) \neq \mu_\bT(\bx)] \le \eps,$ where $\deg(p) \le  \log(s/\eps).$ 
\end{proposition}
In order to handle our challenging noise model, it is important that we can guarantee the $p$ in \Cref{prop:mu-low-degree-bounded} is bounded. Without that guarantee, $L_1$ and $L_2$ regression are not known to handle noise in both the examples and the labels.

\begin{proof}
    For any leaf $\ell$ of $\bT$, let $\mathrm{depth}(\ell)$ be the number of \textit{deterministic} nodes on the root-to-leaf path to $\ell$, not counting $\ell$ itself. The fraction of inputs in $\zo^n$ that have a nonzero chance of reaching $\ell$ is $2^{-\mathrm{depth}(\ell)}$. Now, let $\bT'$ be the stochastic decision tree that is nearly equivalent to $\bT$ except if an input reaches a leaf with deterministic depth more than $\log(s/\eps)$, $\bT'$ returns $0$. We claim that $p \coloneqq \mu_{\bT'}$ satisfies \Cref{prop:mu-low-degree-bounded}. For that, we need to verify three things about $\mu_{\bT'}$:
    \begin{enumerate}
        \item $\mu_{\bT'}$ and $\mu_{\bT}$ are close:
            $\Prx_{\bx}[\mu_\bT(\bx) \neq \mu_{\bT'}(\bx)] \le \eps.$
        This is true because $\bT$ and $\bT'$ can differ only on inputs which reach a leaf with deterministic depth at least $\log(s/\eps)$. At most $2^{-\log(s/\eps)} = \eps/s$ fraction of inputs reach each such leaf, and there are at most $s$ of them.
        \item $\mu_{\bT'}$ is a degree $\log(s/\eps)$ polynomial. We can write $\mu_{\bT'}(x)$ as
        \begin{align*}
            \mu_{\bT'}(x) &= \sum_{\text{leaves } \ell \in \bT'} \Prx[x \text{ reaches }\ell] \cdot (\text{label of $\ell$}) \\
            &= \sum_{\text{leaves } \ell \in \bT} \Prx[x \text{ reaches }\ell] \cdot \Ind[\mathrm{depth}(\ell) \leq \log(s/\eps)] \cdot (\text{label of $\ell$}).
        \end{align*}
        The expression $\Prx[x \text{ reaches }\ell]$ is a degree $\mathrm{depth}(\ell)$ polynomial. Therefore, $\mu_{\bT'}(x)$ is a degree $\log(s/\eps)$ polynomial.
        \item The output of $\mu_{\bT'}$ is bounded on $[0,1]$. This is true since $\bT'$ always returns a value in $\zo$.\qedhere 
    \end{enumerate}
\end{proof}


\subsection{$L_2$ Regression}
Given corrupted samples from some stochastic DT $\bT$, we will apply \Cref{lem:l2 error to mean error}, given below, to show that $L_2$ regression can find a function $f$ that is close to $\mu_{\bT}$. Then, we will apply \Cref{fact:L1-to-Bayes} to generate a hypothesis with error close to the Bayes optimal error.

\begin{lemma}[$L_2$ error to mean error]
\label{lem:l2 error to mean error}
    Fix any stochastic DT $\bT : \zo^n \to \zo$, degree $d \in \N$, and $\eps, \delta > 0$. For a sample size of
    \begin{align*}
        m \coloneqq \poly\left(n^{d}, 1/\eps, \log(1/\delta) \right),
    \end{align*}
    let $\mathcal{S}^\circ$ be a dataset of $m$ i.i.d points of the form $(\bx, \bT(\bx))$. With probability at least $1 - \delta$, there exists a constant $C \in \mathbb{R}$ for which the following holds for all degree $d$ polynomials $p: \zo^n \to \R$.
    \begin{align}
     \label{eq:l2 error to mean error}
        \left| \Ex_{(\bx, \by) \sim \mathcal{S}^\circ}\left[(\trunc(p(\bx)) - \by)^2 \right] - \left(\Ex_{\bx \sim \zo^n}\left[(\trunc(p(\bx)) - \mu_{\bT}(\bx))^2 \right] + C\right) \right| \leq \eps.
    \end{align}
\end{lemma}
\begin{proof}
    We prove \Cref{lem:l2 error to mean error} in two steps: First, we argue that there is a $C$ for which \Cref{eq:l2 error to mean error} holds for any fixed polynomial with extremely high probability. Then, we discretize the set of all truncated degree $d$ polynomials into a finite set $\mathcal{P}$. By union bound, we can show that \Cref{eq:l2 error to mean error} applies to all functions in $\mathcal{P}$, and since every truncated degree $d$ polynomial is sufficiently close to a function in $\mathcal{P}$, this is enough to guarantee that \Cref{eq:l2 error to mean error} applies to all degree $d$ polynomials.
    
    We use the following identity: For any constant $a \in \mathbb{R}$ and random variable $\bz \in \R$,
    \begin{align*}
        \Ex_{\bz}\left[(a - \bz)^2 \right] = (a - \Ex[\bz])^2 + \Var[\bz].
    \end{align*}
    Fix any $p: \zo^n \to \R$. For any $x \in \zo^n$, $\bT(x)$ is a random variable with mean $\mu_{\bT}(x)$. Therefore,
    \begin{align*}
        \Ex_{\bx \sim \zo^n}\left[(\trunc(p(\bx)) - \bT(x))^2 \right] = \Ex_{\bx \sim \zo^n}\left[(\trunc(p(\bx)) - \mu_\bT(x))^2 \right] + \Ex_{\bx \sim \zo^n}\left[\Var[\bT(x)]\right]
    \end{align*}
    For $C = \Ex_{\bx \sim \zo^n}\left[\Var[\bT(x)]\right]$, \Cref{eq:l2 error to mean error} holds in expectation over $\mathcal{S}$ with $\eps = 0$. Since $(\trunc(p(x)) - y)^2$ is bounded on $[0,1]$, we can apply Hoeffdings inequality: For any fixed $p$, \Cref{eq:l2 error to mean error} holds with probability at least $1 - 2 \exp_e(-2m^2 \eps^2)$.
    
    We next discretize the set of all truncated degree $d$ polynomials. Let $\mathcal{P}$ be the following finite set of functions,
    \begin{align*}
        \mathcal{P} \coloneqq \{\trunc \circ p \,|\, \text{$p$ is degree-$d$ polynomial with coefficients that are all a multiple of $\eps/{n^d}$}\} 
    \end{align*}
    Degree $d$ polynomials have at most $n^d$ coefficients. Therefore,
    \begin{align*}
        \log(|\mathcal{P}|) \leq \log\left( \left(\frac{n^d}{\eps} \right)^{n^d} \right) = \poly\left(n^{O(d)}, \log(1/\eps) \right).
    \end{align*}
    This means that for the sample size in \Cref{lem:l2 error to mean error}, \Cref{eq:l2 error to mean error} holds for all functions in $\mathcal{P}$ with probability at least $1 - \delta$. We show that \Cref{eq:l2 error to mean error} holding for function in $\mathcal{P}$ implies the desired result.
    
    Every degree $d$ truncated polynomial is \textit{pointwise} close to a function in $\mathcal{P}$: Fix any degree $d$ polynomial $p$. There is some $f \in \mathcal{P}$, for which
    \begin{align*}
        |\trunc(p(x)) -  f(x)| < \eps \quad \text{for all $x \in \zo^n$}.
    \end{align*}
    This $f$ is easy to specify: It's the truncation of $p'$, where $p'$ is $p$ with all of its coefficients rounded to the nearest $\eps/{n^d}$. In order to expand \Cref{eq:l2 error to mean error} to $p$, we use the following inequality for all $a, \eps \in [0,1]$:
    \begin{align*}
        |\,(a + \eps)^2 - a^2\,| = |2a\eps|+ \eps^2 \leq 3 |\eps|.
    \end{align*}
    Therefore,
    \begin{align*}
         \left|\Ex_{(\bx, \by) \sim \mathcal{S}}\left[(\trunc(p(\bx)) - \by)^2 \right] -  \Ex_{(\bx, \by) \sim \mathcal{S}}\left[f(\bx) - \by)^2 \right]\right| &\leq 3 \max_{x \in \zo^n} \left|\,\trunc(p(x)) - f(x) \,\right| \\
         &\leq 3\,\eps.
    \end{align*}
    Similarly, $\Ex_{\bx \sim \zo^n}\left[(\trunc(p(\bx)) - \mu_{\bT}(\bx))^2 \right]$ and $\Ex_{\bx \sim \zo^n}\left[(f(\bx) - \mu_{\bT}(\bx))^2 \right]$ are within $3\eps$ of one another. Finally, by triangle inequality,
    \begin{align*}
        \left| \Ex_{(\bx, \by) \sim \mathcal{S}}\left[(\trunc(p(\bx)) - \by)^2 \right] - \left(\Ex_{\bx \sim \zo^n}\left[(\trunc(p(\bx)) - \mu_{\bT}(\bx))^2 \right] + C\right) \right| \leq 7\eps.
    \end{align*}
    The desired result holds if we rename $\eps' = \frac{\eps}{7}.$
\end{proof}

We are now ready to prove the low-degree algorithm (i.e. $L_2$ regression) part of \Cref{thm:polynomial regression}.
\begin{lemma}[$L_2$ regression part of \Cref{thm:polynomial regression}]
    Choose any $\eps, \eta, \delta \in (0,1)$, $s \in \N$, and size-$s$ stochastic decision tree $\bT: \zo^n \to \zo$. For a sample of size 
    \begin{align*}
        m \coloneqq \poly\left(n^{O(d)}, \frac{1}{\eps}, \log\left(\frac{1}{\delta}\right) \right),
    \end{align*}
    let $\mathcal{S}$ be an $\eta$-corrupted set of $m$ uniform random samples from $\bT$. If
    \begin{align*}
        p^* = \argmin_{\text{Degree $\log(s/\eps)$ polynomials $p$}} \left( \Ex_{(\bx, \by) \sim \mathcal{S}}\left[(p(\bx) - \by)^2 \right] \right),
    \end{align*}
    and $h:\zo^n \to \zo$ is the hypothesis $h(x) = \round(\trunc(p^*(x)))$. Then with probability at least $1 - \delta$ over the randomness of the sample,
    \begin{align*}
        \error_\bT(h) \leq \opt + O(\sqrt{\eta}) + \eps.
    \end{align*}
\end{lemma}
\begin{proof}
    Let $\mathcal{S}^\circ$ be the original uncorrupted (i.i.d) set of samples, from which $\mathcal{S}$ differs on at most $\eta$ fraction of points. By \Cref{lem:l2 error to mean error}, \Cref{eq:l2 error to mean error} holds, with respect to $\mathcal{S}^\circ$, for all degree $\log (s / \eps)$ polynomials with probability at least $1 - \delta$. We show that if it holds, then $\error_\bT(h) \leq \opt +O(\sqrt{\eta}) + \eps$.
    
    \Cref{prop:mu-low-degree-bounded} guarantees there exists $p: \zo^n \to [0,1]$, a degree $\log(s/\eps)$ bounded polynomial, satisfying
    \begin{align*}
         \Ex_{\bx\sim \zo^n}\left[ (p(\bx)-\mu_\bT(\bx))^2 \right] \le \eps.
    \end{align*}
    Fix $C$ as in \Cref{lem:l2 error to mean error}. Combining \Cref{eq:l2 error to mean error} and \Cref{fact:bound on corruption}, we have that
    \begin{align*}
        \Ex_{(\bx, \by) \sim \mathcal{S}}\left[(p(\bx) - \by)^2 \right] \leq  C + 2\eps + \eta.
    \end{align*}
    Since $p^*$ has the minimum $L_2$ error of all degree $\log(s/\eps)$ polynomials on $\mathcal{S}$,
    \begin{align*}
        \Ex_{(\bx, \by) \sim \mathcal{S}}\left[(p^*(\bx) - \by)^2 \right] \leq  C + 2\eps + \eta.
    \end{align*}
    Truncating $p^*$ can only decrease its $L_2$ error. Combining that with a second application of \Cref{fact:bound on corruption},
    \begin{align*}
        \Ex_{(\bx, \by) \sim \mathcal{S}^\circ}\left[(\trunc(p^*(\bx)) - \by)^2 \right] \leq  C + 2\eps + 2\eta.
    \end{align*}
    Then, by \Cref{eq:l2 error to mean error},
    \begin{align}
        \label{eq:l2 error low}
        \Ex_{\bx \sim \zo^n}\left[(\trunc(p^*(\bx)) - \mu_{\bT}(\bx))^2 \right]  \leq 3\eps + 2\eta.
    \end{align}
    Finally,
    \begin{align*}
        \error_{\bT}(h) &\leq \BayesOpt_\bT + 2\Ex_{\bx \sim \zo^n}[|\,\mu_\bT(\bx) - \trunc(p^*(\bx))\,|] \tag*{\Cref{fact:L1-to-Bayes}} \\
        &\leq \BayesOpt_\bT + 2\sqrt{\Ex_{\bx \sim \zo^n}[(\mu_\bT(\bx) - \trunc(p^*(\bx)))^2]} \tag*{Jensen's inequality}\\
         &\leq \BayesOpt_\bT + 2\sqrt{3\eps + 2\eta} \tag*{\Cref{eq:l2 error low}} \\
        &\leq \BayesOpt_\bT + O(\sqrt{\eps}) + O(\sqrt{\eta}).
    \end{align*}
    The desired result then holds by renaming $\eps' = \Omega(\eps^2)$.
\end{proof}

\subsection{$L_1$ regression}

We will need the following generalization bound:
\begin{lemma}[$L_1$ error generalization]
\label{lem:l1 error generalization}
    Fix any stochastic DT $\bT : \zo^n \to \zo$, degree $d \in \N$, and $\eps, \delta > 0$. For a sample size of
    \begin{align*}
        m \coloneqq \poly\left(n^{O(d)}, \frac{1}{\eps}, \log\left(\frac{1}{\delta}\right) \right),
    \end{align*}
    let $\mathcal{S}^\circ$ be a dataset of $m$ i.i.d points of the form $(\bx, \bT(\bx))$. With probability at least $1 - \delta$, the following holds for all degree $d$ polynomials $p: \zo^n \to \R$.
    \begin{align}
        \label{eq:l1 error generalization}
        \left| \Ex_{(\bx, \by) \sim \mathcal{S}^\circ}\left[|\trunc(p(\bx)) - \by|^2 \right] - \Ex_{\bx \sim \zo^n, \bT}\left[|\trunc(p(\bx)) - \bT(\bx)|\right] \right| \leq \eps.
    \end{align}

\end{lemma}
\Cref{lem:l1 error generalization} can be proven using the same discretization argument as \Cref{lem:l2 error to mean error}. We omit the proof for brevity.

\begin{lemma}[$L_1$ regression part of \Cref{thm:polynomial regression}]
\label{lem:l1 regression}
    Choose any $\eps, \eta, \delta \in (0,1)$, $s \in \N$, and size-$s$ stochastic decision tree $\bT: \zo^n \to \zo$. For a sample of size 
    \begin{align*}
        m \coloneqq \poly\left(n^{O(d)}, \frac{1}{\eps}, \log\left(\frac{1}{\delta}\right) \right),
    \end{align*}
    let $\mathcal{S}$ be an $\eta$-corrupted set of $m$ uniform random samples from $\bT$. If
    \begin{align*}
        p^* = \argmin_{\text{Degree $\log(s/\eps)$ polynomials $p$}} \left( \Ex_{(\bx, \by) \sim \mathcal{S}}\left[|\,p(\bx) - \by)\,| \right] \right),
    \end{align*}
    and $\bh:\zo^n \to \zo$ is the randomized hypothesis where $\bh(x)$ is $1$ with probability $\trunc(p^*(x))$ and $0$ otherwise. Then with probability at least $1 - \delta$ over the randomness of the sample,
    \begin{align*}
        \error_\bT(\bh) \leq 2\opt + 2\eta + \eps.
    \end{align*}
\end{lemma}

\begin{proof}
    Let $\mathcal{S}^\circ$ be the original uncorrupted (i.i.d) set of samples from which $\mathcal{S}$ differs on at most $\eta$ fraction of points. By \Cref{lem:l1 error generalization}, \Cref{eq:l1 error generalization} holds, with respect to $\mathcal{S}^\circ$ for all degree $\log (s / \eps)$ polynomials with probability at least $1 - \delta$. We show that if it holds, then $\error_\bT(h) \leq 2\opt + 2\eta + \eps$.

    \Cref{prop:mu-low-degree-bounded} guarantees there exists $p: \zo^n \to [0,1]$, a degree $\log(s/\eps) $ polynomial, satisfying
   \begin{align*}
         \Ex_{\bx}\left[ \left|p(\bx)-\mu_\bT(\bx)\right| \right] \le \eps.
    \end{align*}
    We first bound the expected error of $\mu_\bT(\bx)$ relative to $\bT(\bx)$.
    \begin{align*}
        \Ex_{\bx}\left[\left|\mu_\bT(\bx) - \bT(\bx)\right|\right] &= \Ex_{\bx}\left[\Pr[\bT(\bx) = 1](1 - \mu_\bT(\bx)) + \Pr[\bT(\bx) = 0](\mu_\bT(\bx))\right]  \\
        &=\Ex_{\bx}\left[2 \mu_\bT(\bx)(1-\mu_\bT(\bx))\right]\\
        &\leq 2\cdot \Ex_{\bx}\left[\min(\mu_\bT(\bx), 1-\mu_\bT(\bx))\right] \\
        &= 2\cdot \BayesOpt_{\bT}
    \end{align*}
    By triangle inequality, we have that $\Ex_{\bx}[|p(\bx) - \bT(\bx)|] \leq 2\cdot \BayesOpt_{\bT} + \eps$. By \Cref{eq:l1 error generalization}
    \begin{align*}
        \Ex_{\bx, \by \sim \mathcal{S}^\circ}[|p(\bx) - \by|] \leq 2\cdot \BayesOpt_{\bT} + 2\eps.
    \end{align*}
    By \Cref{fact:bound on corruption} initialized with $\mathrm{err}(x,y) = |x - y|$,
    \begin{align*}
        \Ex_{\bx, \by \sim \mathcal{S}}[|p(\bx) - \by|] \leq 2\cdot \BayesOpt_{\bT} + 2\eps + \eta.
    \end{align*}
    Since $p^*$ has minimum $L_1$ error among all degree $\log(s/\eps)$ polynomials,
    \begin{align*}
        \Ex_{\bx, \by \sim \mathcal{S}}[|p^*(\bx) - \by|] \leq 2\cdot \BayesOpt_{\bT} + 2\eps + \eta.
    \end{align*}
    Reapplying \Cref{fact:bound on corruption}, combined with the fact that truncating $p^*$ can only decrease its error,
    \begin{align*}
        \Ex_{\bx, \by \sim \mathcal{S}^\circ}[|\trunc(p^*(\bx)) - \by|] &\leq  \Ex_{\bx, \by \sim \mathcal{S}}[|\trunc(p^*(\bx)) - \by|] + \eta \\
        &\leq 2\cdot \BayesOpt_{\bT} + 2\eps + 2\eta.
    \end{align*}
    Applying \Cref{eq:l1 error generalization} again.
    \begin{align*}
        \Ex_{\bx \sim \zo^n}[|p(\bx) - \bT(\bx)|] &\leq \Ex_{\bx, \by \sim \mathcal{S}^\circ}[|p(\bx) - \bT(\bx)|] + \eps \\
        &\leq 2\cdot \BayesOpt_{\bT} + 3\eps + 2\eta.
    \end{align*}
    Finally, since $\bh(x)$ returns $1$ with probability $\trunc(p^*(x))$, and $\bT(x)$ is always in $\zo$,
    \begin{align*}
        \Prx_{\bx, \bh, \bT}[\bh(\bx) \neq \bT(\bx)] &=  \Ex_{\bx \sim \zo^n}[|p(\bx) - \bT(\bx)|] \\ 
         &\leq 2\cdot \BayesOpt_{\bT} + 3\eps + 2\eta.
    \end{align*}
     The desired result holds with the renaming $\eps' = \frac{\eps}{3}$.
\end{proof}

\section*{Acknowledgements}

We thank the anonymous reviewers for their detailed and helpful feedback.  

\bibliography{most-influential}{}
\bibliographystyle{alpha}

\end{document}